\documentclass[11pt]{article}
\usepackage[T1]{fontenc}
\usepackage{lmodern}
\usepackage{tgpagella}
\usepackage{sectsty}
\allsectionsfont{\sffamily}
\usepackage{titling}

\usepackage{graphicx,amsmath,amsfonts,amsthm,amscd,amssymb,bm,url,color,latexsym}

\renewenvironment{proof}[1][\proofname]{\noindent {\bfseries #1.}  }{\qed}
\usepackage[margin=1in]{geometry}
\usepackage[labelfont=bf, font=sf, margin=1em]{caption}
\usepackage{subcaption}
\usepackage{bbm}
\usepackage{microtype}
\usepackage{algorithm}
\usepackage{algorithmicx}
\usepackage{algpseudocode}
\usepackage{verbatim}
\usepackage{framed}
\usepackage{comment}
\usepackage{tikz}
\usepackage{fge}
\usepackage{mathtools}
\usepackage{enumerate}
\usepackage{dsfont}
\usepackage{booktabs}
\usepackage{multirow}
\newcommand\bluebf[1]{\textcolor{blue}{\mathbf{$#1$}}}
\usepackage{array}
\makeatletter
\newcommand{\thickhline}{%
    \noalign {\ifnum 0=`}\fi \hrule height 1.5pt
    \futurelet \reserved@a \@xhline
}
\newcolumntype{"}{@{\hskip\tabcolsep\vrule width 1.5pt\hskip\tabcolsep}}
\makeatother

\usepackage{mathrsfs}
\usepackage{physics}

\allowdisplaybreaks

\usepackage{hyperref}
\hypersetup{
    colorlinks=true,%
    citecolor=blue,%
    filecolor=blue,%
    linkcolor=blue,%
    urlcolor=blue
}
\usepackage[toc, page]{appendix}

\usepackage[capitalize,nameinlink]{cleveref}

\newtheorem{theorem}{Theorem}[section]

\newtheorem{proposition}[theorem]{Proposition}
\newtheorem{definition}[theorem]{Definition}

\renewcommand{\mathbf}{\boldsymbol}

\newcommand{\mb}{\mathbf}
\newcommand{\mc}{\mathcal}

\newcommand{\md}{\mathds}
\newcommand{\bb}{\mathbb}

\newcommand{\set}[1]{\left\{ #1 \right\}}

\newcommand{\reals}{\bb R}

\newcommand{\eps}{\varepsilon}
\newcommand{\R}{\reals}
\newcommand{\Cp}{\bb C}

\newcommand{\paren}{\pqty}

\newcommand{\e}{\mathrm{e}}

\newcommand{\wt}{\widetilde}

\numberwithin{equation}{section}

\title{Inverse Problems, Deep Learning, and Symmetry Breaking}
\author{Kshitij Tayal\thanks{Department of Computer Science and Engineering, University of Minnesota, Twin Cities. Email: \href{mailto:tayal007@umn.edu}{\texttt{tayal007@umn.edu}}}
\and
Chieh-Hsin Lai\thanks{School of Mathematics, University of Minnesota, Twin Cities. Email: \href{mailto:laixx313@umn.edu}{\texttt{laixx313@umn.edu}}}
\and
Vipin Kumar\thanks{Department of Computer Science and Engineering, University of Minnesota, Twin Cities. Email: \href{mailto:kumar001@umn.edu}{\texttt{kumar001@umn.edu}}}
\and
Ju Sun\thanks{Department of Computer Science and Engineering, University of Minnesota, Twin Cities. Email: \href{mailto:jusun@umn.edu}{\texttt{jusun@umn.edu}}}
  }

\date{\today}

\begin{document}
\maketitle

\begin{abstract}
In many physical systems, inputs related by intrinsic system symmetries are mapped to the same output. When inverting such systems, i.e., solving the associated inverse problems, there is no unique solution. This causes fundamental difficulties for deploying the emerging end-to-end deep learning approach. Using the generalized phase retrieval problem as an illustrative example, we show that careful symmetry breaking on the training data can help get rid of the difficulties and significantly improve the learning performance. We also extract and highlight the underlying mathematical principle of the proposed solution, which is directly applicable to other inverse problems.
\end{abstract}

\section{Introduction}\label{Sec:Intro}

\subsection{Inverse problems and deep learning}
For many physical systems, we observe only the output and strive to infer the input. The inference task is captured by the umbrella term inverse problem. Formally, the underlying system is modeled by a forward mapping $f$, and solving the inverse problem amounts to identifying the inverse mapping $f^{-1}$; see \cref{fig:inv_end2end}. Inverse problems abound in numerous fields and take diverse forms: structure from motion in computer vision~\cite{HartleyZisserman2003Multiple}, image restoration in image processing~\cite{GonzalezWoods2017Digital}, source separation in acoustics~\cite{Comon2010Handbook}, inverse scattering in physics~\cite{ColtonKress2013Inverse}, tomography in medical imaging~\cite{Herman2009Fundamentals}, soil profile estimation in remote sensing~\cite{entekhabi1994solving}, various factorization problems in machine learning~\cite{Ge2013Provable}, to name a few.

Let $\mb y$ denote the observed output. Traditionally, inverse problems are mostly formulated as optimization problems of the form
\begin{align} \label{eq:inverse_opt}
\min_{\mb x} \; \ell\paren{\mb y, f\paren{\mb x}} + \lambda \Omega\paren{\mb x},
\end{align}
where $\mb x$ represents the input to be estimated. In the formulation, $\ell\paren{\mb y, f\paren{\mb x}}$ ensures that $\mb y \approx f\paren{\mb x}$ ($\ell$ means loss) in an appropriate sense, $\Omega\paren{\mb x}$ encodes prior knowledge about $\mb x$---often the input cannot be uniquely determined from the observation $\mb y$ alone (i.e., the problem is \emph{ill-posed}) and the knowledge-based regularization $\Omega\paren{\mb x}$ may help mitigate the issue and make the problem well-posed, and $\lambda$ is a tradeoff parameter. For simple inverse problems, \cref{eq:inverse_opt} may admit a simple closed-form solution. For general problems, iterative numerical optimization algorithms are often developed to solve \cref{eq:inverse_opt}~\cite{Kirsch2011Introduction}.

The advent of deep learning has brought tremendous novel opportunities for solving inverse problems. For example, one can learn data-driven loss term $\ell$ and regularization term $\Omega$, and one can also replace components of iterative methods for solving \cref{eq:inverse_opt} by trained neural networks.
\begin{figure}[!htbp]
	\centering
	\includegraphics[width = 0.6\linewidth]{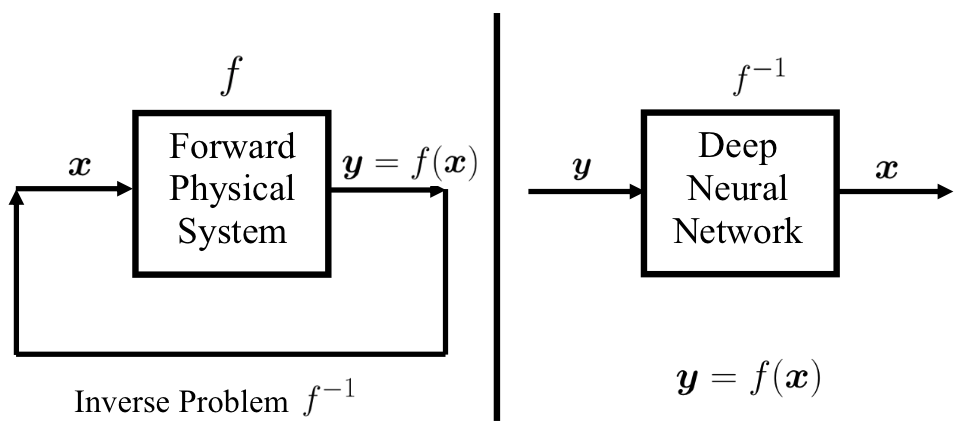}
	\vspace{-.1in}
	\caption{Illustration of inverse problem and the end-to-end learning approach.}
	\label{fig:inv_end2end}
\end{figure}
 These ideas has enabled capturing structures in practical data that are traditionally difficult to encode using analytic expressions, and lead to faster and/or more effective algorithms. The most radical is perhaps the end-to-end approach: a deep neural network (DNN) is directly set up and trained to approximate the inverse mapping $f^{-1}$---backed by the famous universal approximation theorem~\cite{PoggioEtAl2017Why}---based on a sufficiently large set of $\paren{\mb x, \mb y}$ pairs; see \cref{fig:inv_end2end}. Instead of citing the abundance of individual papers, we refer the reader to the excellent review articles~\cite{McCannEtAl2017Convolutional,LucasEtAl2018Using,ArridgeEtAl2019Solving} on these developments.

\subsection{Difficulty with symmetries}   \label{sec:intro_sym_problem}

In this paper, we focus on the end-to-end learning approach for inverse problems. This approach has recently been widely acclaimed for its remarkable performance on several tasks such as image denoising~\cite{xie2012image}, image super-resolution~\cite{dong2014learning}, image deblurring~\cite{xu2014deep}, and sparse recovery~\cite{mousavi2017learning}. These problems are all linear inverse problems, for which the forward mapping $f$ is linear. What about nonlinear problems?

When the forward mapping $f$ is nonlinear, we start to see intrinsic symmetries in many systems. We give several quick examples here:
\begin{itemize}
\item \textbf{Fourier phase retrieval}~\cite{BendoryEtAl2017Fourier} The forward model is $\mb Y = \abs{\mc F \paren{\mb X}}^2$, where $\mb X \in \Cp^{n \times n}$ and $\mb Y \in \R^{m \times m}$ are matrices and $\mc F$ is a $2$D oversampled Fourier transform. The operation $\abs{\cdot}$ takes complex magnitudes of the entries elementwise. It is known that translations and conjugate flippings applied on $\mb X$, and also global phase transfer of the form $e^{i \theta } \mb X$ all lead to the same $\mb Y$.
\item \textbf{Blind deconvolution}~\cite{LamGoodman2000Iterative,TonellotBroadhead2010Sparse} The forward model is $\mb y = \mb a \circledast  \mb x$, where $\mb a$ is the convolution kernel, $\mb x$ is the signal (e.g., image) of interest, and $\circledast$ denotes the circular convolution. Both $\mb a$ and $\mb x$ are inputs. Here, $\mb a \circledast \mb x = \paren{\lambda \mb a} \circledast \paren{\mb x/\lambda}$ for any $\lambda \ne 0$, and circularly shifting $\mb a$ to the left and shifting $\mb x$ to the right by the same amount does not change $\mb y$.
\item \textbf{Blind source separation}~\cite{Comon2010Handbook} The forward model is $\mb Y = \mb A \mb X$, where $\mb A$ is the mixing matrix and $\mb X$ is the source matrix and both $\mb A$ and $\mb X$ are inputs. The scaling symmetry similar to above is also present here. Moreover, signed permutations are another kind of symmetry, i.e., $\mb A \mb X =\paren{ \mb A \mb \Pi \mb \Sigma} \paren{\mb \Sigma^{-1} \mb \Pi^{-1} \mb X }$ for any permutation matrix $\mb \Pi$ and any diagonal sign matrix $\mb \Sigma$. \footnote{For both blind deconvolution and blind source separation, depending on structures of the inputs, there may be other symmetries that we have not covered here. The symmetries we have discussed tend to be persistent nonetheless.}
\item \textbf{Synchronization over compact groups}~\cite{PerryEtAl2018Message} For $g_1, \dots, g_n$ over a compact group $\mc G$, the observation is a set of pairwise relative measurements $y_{ij} = g_i g_j^{-1}$ for all $\paren{i, j}$ in an index set $\mc E \subset \set{1, \dots, n} \times \set{1, \dots, n}$. Obviously, any global shift of the form $g_k \mapsto g_k g$ for all $k \in \set{1, \dots, n}$, for any $g \in \mc G$, leads to the same set of measurements.
\end{itemize}
Solving these inverse problems means recovering the input up to the intrinsic system symmetries, as evidently this is the best one can hope for.

Are symmetries good or bad for problem solving? That depends on how we deal with them.
\begin{figure}[!htbp]
	\centering
	\includegraphics[width=0.2\linewidth]{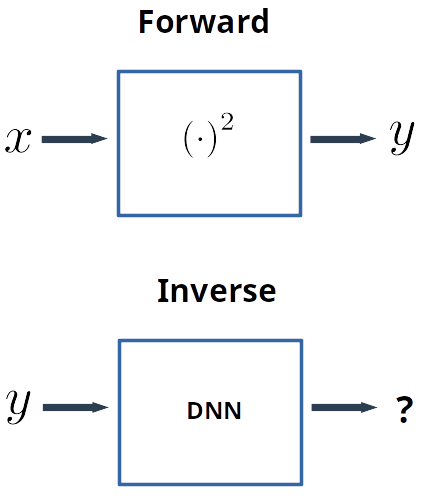}
	\hspace{1em}
	\includegraphics[width=0.4\linewidth]{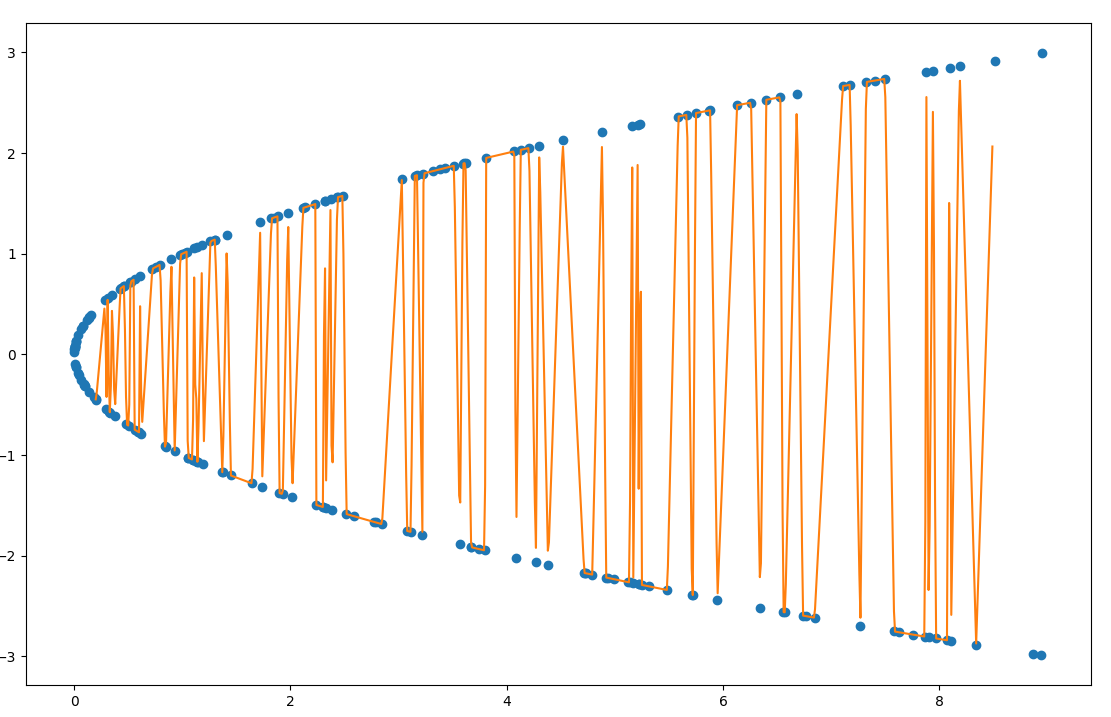}
	\caption{Learn to take square root. (Left) The forward and inverse models;  (Right) The function (in orange) determined by the training points. }
	\label{fig:learn_sqrt}
\end{figure}
Imagine that we are bored and try to train a DNN to take square root. We randomly draw sufficiently many $x \in \R$ and thereby generate training samples $\set{\paren{x_i, x_i^2}}$. We feed the data samples to the DNN and push the button to start torturing our GPU machines. We are happy so long as ultimately the trained DNN output a good estimate of the square root for any input, up to sign. Does it work out as expected? It turns out not quite. For any two points $x_i^2$ and $x_j^2$ that are near, the corresponding $x_i$ and $x_j$ may be close in magnitude but differ in sign. This implies that the function determined by the training data points is highly oscillatory (see \cref{fig:learn_sqrt}) and behaves like a function with many and frequent points of discontinuity---Interestingly, the more train samples one gathers, the more serious the problem is. Intuitively, DNNs with continuous or even smooth activation functions struggle when approximating these irregular functions.

\subsection{Our contribution: symmetry breaking}

The above example is of course contrived and an easy fix for the problem is we only take positive (or negative) $x_i$'s. For general inverse problems with symmetries, so long as the symmetries can relate remote inputs to the same output, e.g., all the symmetries we discussed in the quick examples, the above issue of approximating highly irregular functions arises. It is a natural question if our easy fix for learning square root can be generalized. In this paper,
\begin{itemize}
	\item We take the generalized phase retrieval problem as an example, and show that effective symmetry breaking can be performed for both the real-valued and complex-valued versions of the problem. We also corroborate our theory with extensive numerical experiments.
	\item By working out the example, we identify the basic principle of effective symmetry breaking, which can be readily applied to other inverse problems with symmetries.
\end{itemize}

\paragraph{Notation.} Boldface capitals are matrices (e.g., $\mb A$) and boldface letters are vectors (e.g., $\mb x$). $\R_+$ means the set of positive reals. Measure by default is the Lebesgue measure for Euclidean spaces. Other notations are either standard or defined inline.

\section{Generalized Phase Retrieval}
\label{sec:gpr}
Fourier phase retrieval (PR) alluded to above is a classical problem in computational imaging with a host of applications~\cite{shechtman2015phase,BendoryEtAl2017Fourier}. Despite the existence of effective empirical methods to solve the problem~\cite{Fienup1982Phase}, there is little theory on when and why these methods work and also if there are alternatives.

Over the past decade, numerous papers from the signal processing and applied mathematics communities have tried to develop provable methods for PR, based on generalized models in which the Fourier transform $\mc F$ is replaced by a generic, often random, linear operator.  This is the generalized PR problem. The name arguably is a misnomer. When $\mc F$ is replaced by a random linear operator, the randomness often helps kill the translation and conjugate flipping symmetries in Fourier PR and only the global phase symmetry is left. So the generalization often leads to simplified, if not over, PR problems.

Nonetheless, the focus of this paper is not on solving (Fourier) PR per se, but on demonstrating the principle of symmetry-breaking taking generalized PR as an example.  We work with two versions of generalized PR:
\begin{description}
\item [Real Gaussian PR] The forward model: $\mb y = \abs{\mb A \mb x}^2$, where $\mb x \in \R^n$, $\mb y \in \R^m$, and $\mb A \in \R^{m \times n}$ is iid real Gaussian. The absolute-square operator $\abs{\cdot}^2$ is applied elementwise. The only symmetry is sign, as $\mb x$ and $-\mb x$ are mapped to the same $\mb y$.
\item [Complex Gaussian PR] The forward model: $\mb y = \abs{\mb A \mb x}^2$, where $\mb x \in \Cp^n$, $\mb y \in \R^m$, and $\mb A \in \Cp^{m \times n}$ is iid complex Gaussian. The modulus-square operator $\abs{\cdot}^2$ is applied elementwise. The only symmetry is global phase shift, as $e^{i \theta}\mb x$ for all $\theta \in [0, 2\pi)$ are mapped to the same $\mb y$.
\end{description}
These two versions have been intensively studied in the recent developments of generalized PR; see, e.g., ~\cite{CandesEtAl2012PhaseLift,CandesEtAl2015Phase,SunEtAl2017Geometric}.

\subsection{Real Gaussian PR}
\begin{figure}[!htbp]
  \centering
  \includegraphics[width=0.6\linewidth]{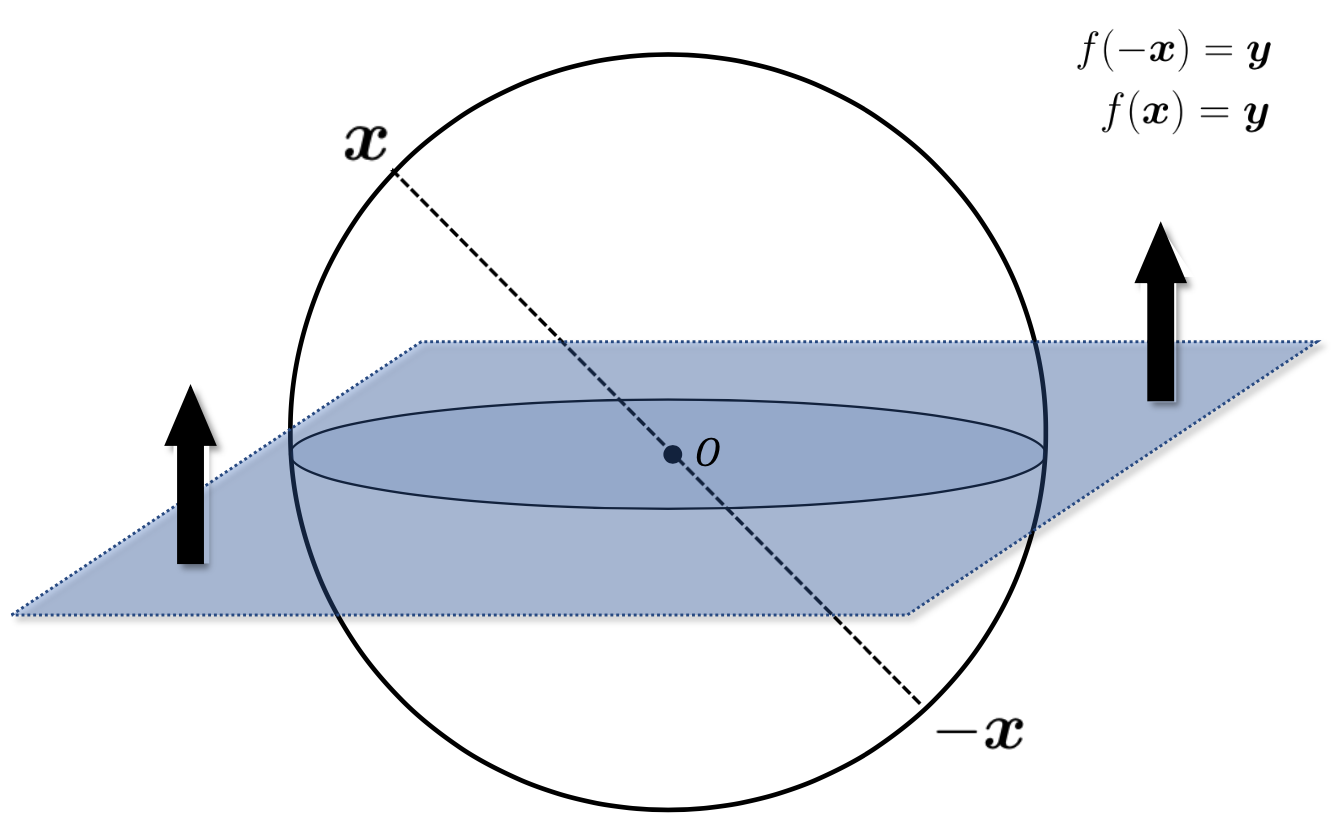}
  \caption{Symmetry breaking for real Gaussian phase retrieval. }
  \label{fig:sphere_cut}
\end{figure}
Recall that in our learning square root example, the sign ambiguity caused the irregularity in the function determined by the training samples. A similar problem occurs here. For two samples that are close in the observation, say $\mb y$ and $\mb y + \mb \delta$ for a small $\mb \delta$, the paired inputs may be $\mb x$ and $-\paren{\mb x + \mb \delta'}$ for a small $\mb \delta'$. Thus, for the inverse mapping that our DNN tries to approximate, a small perturbation $\mb \delta$ in the variable leads to $2\mb x - \mb \delta'$ change to the function value, and sharp changes of this kind happen frequently as we have many data samples.

It is tempting to generalize our solution to the square root example. There, the symmetry is the sign and we broke it by restricting the range of desired DNN output to $\R_+$. Here, the symmetry is the global sign of vectors and antipodal points map to the same observation. Thus, an intuitive generalization is to break antipodal point pairs, and a simple solution is to make a hyperplane cut and take samples from only one side of the hyperplane! This is illustrated in \cref{fig:sphere_cut} and we use the hyperplane $\set{\mb x \in \R^n: x_n = 0}$---in $\R^3$, this is the $xy$\nobreakdash-plane.

In $\R^3$, we can see directly from~\cref{fig:sphere_cut} that the upper half space cut out by the $xy$\nobreakdash-plane is connected. Moreover, it is representative as any point in the space (except for the plane itself) can be represented by a point in this set by appropriate global sign adjustment, and it cannot be made smaller to remain representative. The following proposition says that these properties also hold for high-dimensional spaces.
\begin{proposition}\label{prop:realsmallest}
Let
\begin{align}
    & R \doteq \left \{ \mb x = ( x_1, \cdots, x_n ) \in \R ^n: x_n > 0  \right \},\\
&Z  \doteq \left \{ \mb x = ( x_1, \cdots, x_n ) \in \R ^n: x_n = 0  \right \}.
\end{align}
Then the following properties hold:
\begin{enumerate}[{(i)}]
    \item \label{prop:1} \textbf{(connected)} $R$ is connected in $\R^n$;
    \item \label{prop:3} \textbf{(representative)} $Z$ is of measure zero \cite{rudin2006real} and for any $\mb x \in \R ^n \setminus  Z,$ either $\mb x\in R $ or $-\mb x\in R$. That is, $R$ can represent any point in $\R^n \setminus Z$;
    \item \label{prop:4} \textbf{(smallest)} If we remove any single point $\mb x$ of $R$, then there is no other point in $R \setminus \{\mb x\}$ that can represent $\mb x$. Namely, the resulting set is not representative anymore.
\end{enumerate}
\end{proposition}
\begin{proof}
First recall the property that if any two points in a given set can be connected by a continuous path lying entirely in the set, then this set must be a connected set \cite{kelley2017general}. Now any two points $\mb x, \mb y \in R$ can be connected by the line segment $\set{\alpha x + \paren{1 - \alpha} \mb y: \alpha \in [0, 1]} \subset R$. Thus $R$ is connected.

Moreover, $Z = \R^{n-1}\times \{ 0 \}$ has Lebesgue measure zero since
\begin{equation}\label{eq:zero}
  \mu \paren{Z} = \displaystyle \int_{\R^n} \mathds{1}_{Z} \; d \mb x=  \displaystyle \int_{\R^{n-1}} \displaystyle\paren{ \int_{ \{ 0 \} } \mathds{1}_{Z}\; dx_n} \; d \mb x_{-n} = 0.
\end{equation}
Here $\mathds{1}_{Z}$ is the indicator function on $Z$, and $\mb x_{-n} \in \R^{n-1}$ is the vector formed by the first $n-1$ coordinates of $\mb x$. We used Tonelli's theorem \cite{rudin2006real} to obtain the second equality, and the fact $\int_{\set{0}} \md 1_Z \; dx_n = 0$ to obtain the third equality. The rest of \eqref{prop:3} is straightforward.

For \eqref{prop:4}, suppose that there is another point $\wt{\mb x} \in R \setminus \{\mb x \}$ which can represent $\mb x$ up to a global sign flipping. Since both $\mb x$ and $\wt{\mb x}$ are in $R$, which means they need to have the same sign for the last component, it must be $\mb x = \mb \tilde{\mb x}$. We get a contradiction.
\end{proof}

The coordinate hyperplane $Z$ we use is arbitrary, and we can prove similar results for any hyperplane passing through the origin. The set $Z$ is negligible, as the probability of sampling a point exactly from $Z$ is zero. In fact, we can break the symmetry in $Z$ also by recursively applying the current idea. For the sake of simplicity and in view of the probable diminishing return, we will not pursue the refined scheme here.

Does this help solve our problem? Imagine that we have collected a set of training samples $\{\mb x_i, \abs{\mb A \mb x_i}^2\}$ for real Gaussian PR. Now we are going to preprocess the data samples according to the above hyperplane cut: for all $\mb x_i$'s, if $\mb x_i$ lies above $Z$, we simply leave it untouched; if $\mb x_i$ lies below $Z$, we switch the sign of $\mb x_i$; if $\mb x_i$ happens to lie on $Z$, we make a small perturbation to $\mb x_i$ and then adjusts the sign as before accordingly. Now $\mb x_i \in  R$ for all $i$. Since $R$ is a connected set, when there are sufficiently dense training samples, small perturbations to $\abs{\mb A \mb x}^2$ always only lead to small perturbations to $\mb x_i$. So we now have a nicely behaved target function to approximate using a DNN. Also, $R$ being representative implies that a sufficiently dense sample set should enable reasonable learning.

The set of three properties is also necessary for effective symmetry breaking and learning. Being representative is easy to understand. If the representative set is not the smallest, symmetry is still present for certain points in the set and so symmetry breaking is not complete. Now the set can be smallest representative but not connected. An example in the setting of \cref{prop:realsmallest} would be taking out a small strict subset of $R$, say $B \subsetneq R$, and consider the set $M \doteq \paren{-B} \cup \paren{R \setminus B}$. It is easy to verify that $M$ is smallest representative, but not connected. This leaves us the trouble of approximating (locally) highly oscillatory functions.

\subsection{Complex Gaussian PR}
We now move to the complex case and deal with a different kind of symmetry. Recall that in the complex Gaussian PR, $e^{i \theta} \mb x$ for all $\theta \in [0, 2\pi)$ are mapped to the same $\abs{\mb A \mb x}^2$, i.e., global phase shift is the symmetry. These ``equivalent" points form a continuous curve in the complex space, contrasting the isolated antipodal point pairs in the real case.

Inspired by the real version, we generalize the three desired properties for symmetry breaking to the complex case. Particularly, ``representative" in this context means:
\begin{definition}[representative] \label{def:represent}
Let $S$ be a subset of $\Cp^n$. We say that $S$ is a representative subset for $\Cp^n$ if the following holds: there is a measure zero subset $Z$ of $\Cp^n$ such that for any $\mb x \in \Cp ^n \setminus  Z,$ we can find a $\theta \in [0, 2\pi)$ and an $\mb x' \in S$ so that $ \mb x =\e^{i \theta}\mb x'$.
\end{definition}
In words, a subset $S$ is representative in this context if except for a negligible subset of $\Cp^n$, any element of $\Cp^n$ can be represented by an element of $S$ after appropriate global phase shift.
\begin{definition}[smallest representative] \label{def:smallest_represent}
Let $S$ be a subset of $\Cp^n$. We say that $S$ is a smallest representative subset for $\Cp^n$ if it is representative and no element in $S$ can be represented by a distinct element of $S$.
\end{definition}
To construct a smallest representative set for $\Cp^n$, it is helpful to start with low dimensions. When $n=1$, any ray stemming from the origin (with origin removed) is a smallest representative subset for $\Cp$. For simplicity, we can take the positive axis $\R_+$. When $n=2$, it is natural to use the building block $\R_+$ for $\Cp$ and start to consider product constructions of the form $\R_+ \times T \subset \Cp^2$ with $T \in \Cp$. Similarly for high dimensions, we try constructions of the form $\R_+ \times T \subset \Cp^n$ with $T\in \Cp^{n-1}$. Another consideration is the measure-zero set. In the real case, we used a coordinate hyperplane. Here, as a natural generalization, we take a complex hyperplane:
\begin{align}  \label{eq:complex_zero_net}
  Z = \left \{ \mb x = ( x_1, \cdots, x_n ) \in \Cp^n: x_1 = 0  \right \}.
\end{align}
The question now is how to choose $T$ to make $\R_+ \times T$ a smallest representative subset for $\Cp^n \setminus Z$.

It turns out we actually do not get many choices. The following result says that real positivity assumed for the first coordinate constrains the construction significantly and the rest of coordinates are forced to be the entire complex space $\Cp^{n-1}$.
\begin{proposition}\label{prop:construct}
If $S \doteq \R_+ \times T$ with $T \subset \Cp^{n-1}$ is a representative subset for $\Cp^n \setminus Z$, then $T=\Cp^{n-1}$.
\end{proposition}
\begin{proof}
We prove by contradiction. Suppose that there is a $\mb x' \in \Cp^{n-1}$ but $\mb x' \notin T$. Then for any $x_1 \in \R_+$, $\mb x = (x_1; \mb x') \notin \R_+ \times T = S$ and $\mb x \in \Cp^{n-1} \setminus Z$. Since $S$ is representative, we can find a $\theta \in [0, 2\pi)$ and $\widetilde{\mb x}\in S$ so that
\begin{equation}\label{eq:represent}
    \e^{i\theta}\mb x = \widetilde{\mb x}.
\end{equation}
 Since $S$ has the first coordinate to be positive real numbers, by looking at the first component of~\cref{eq:represent} we have
\begin{equation} \label{eq:construct_angle}
    \begin{cases}
    x_1\cos\theta > 0 \\
    x_1\sin\theta = 0
    \end{cases},
\end{equation}
from where we deduce that $\theta=0$ and so $\mb x = \widetilde{\mb x} \in S$. This contradicts our construction that $\mb x \notin S$.
\end{proof}

We now focus on this candidate set
\begin{align}  \label{eq:complex_small_rep}
 R \doteq \left \{ \mb x = ( x_1, \cdots, x_n ) \in \Cp ^n: \Im(x_1) = 0, x_1 > 0  \right \}.
\end{align}
Our next proposition confirms that this is indeed a good choice.
\begin{proposition}\label{prop:smallest}
The set $R$ defined in~\cref{eq:complex_small_rep} is a connected, smallest representative set for $\Cp^n \setminus Z$ with $Z$ defined as in~\cref{eq:complex_zero_net}. Moreover, $Z$ is a measure-zero subset of $\Cp^n$.
\end{proposition}
\begin{proof}
First, $Z$ has measure zero due to the same reason as in~\cref{eq:zero}. Next, it is clear that any two points $\mb x, \mb y \in R$ can be connected by the line segment $\set{\alpha \mb x + \paren{1-\alpha} \mb y: \alpha \in [0, 1]} \subset R$, and so $R$ is a connected set. To see $R$ is representative, for any $\mb x = (r_1 e^{i\theta_1}, x_2, \dots, x_n) \in \Cp^n \setminus Z$ where $r_1 > 0$, one can choose $\theta = 2\pi - \theta_1$ so that $e^{i \theta} \mb x \in R$. To show it is also smallest, we use a similar argument to that in \cref{prop:construct}.
Let $\mb x \in R$ where we write $\mb x=(x_1; \mb x')$ with $\mb x' \in \Cp^{n-1}$. If another element $\widetilde{\mb x} \neq \mb x \in R$ can be represented by $\mb x$, namely, if there is $\theta \in [0, 2\pi)$ such that $\widetilde{\mb x} = \e^{i \theta}\mb x$, then we need to have $\Im(\e^{i \theta} x_1)=0$ and $\Re(\e^{i \theta} x_1)>0$. That is,
\begin{equation} \label{eq:angle}
	\begin{cases}
    x_1 \cos \theta > 0\\
    x_1 \sin \theta = 0
    \end{cases}.
\end{equation}
Since $x_1>0$, \cref{eq:angle} implies that $\theta = 0$. But this contradicts with that $\mb x \neq \widetilde{\mb x}$ and thus no element in $R$ can be represented by a distinct element in $R$.
\end{proof}

So our construction $R$ enjoys the three desired properties, similar to the real case, despite that the problem symmetry is different here. Once we emulate the data preprocessing step for the real case, i.e., all $x_i$'s for the training data points $\{(\mb x_i, \abs{\mb A \mb x_i}^2)\}$ are mapped into $R$,
we obtain an effective symmetry breaking algorithm for complex Gaussian PR.

For general inverse problems, although the symmetries might be very different than here and the sample spaces could also be diverse, the three properties we have identified, which concern only the geometric and topological aspects of the space, can potentially generalize as a basic principle for effective symmetry breaking.

\section{Related Work}
As alluded to above, recently there have been intensive research efforts on solving inverse problems using deep learning~\cite{McCannEtAl2017Convolutional,LucasEtAl2018Using,ArridgeEtAl2019Solving}. The end-to-end approach is attractive not only because of its simplicity, but also because (i) we do not even need to know the forward models, so long as we can gather sufficiently many data samples and weak system properties such as symmetries---e.g., this is handy for complex imaging systems~\cite{HorisakiEtAl2016Learning,LiEtAl2018Deep};
(ii) or alternatives have rarely worked, and a good example is Fourier PR~\cite{Fienup1982Phase,SinhaEtAl2017Lensless}.

Besides the linear inverse problems, the end-to-end deep learning approach has been empirically applied to a number of problems with symmetries, e.g., blind image deblurring (i.e., blind deconvolution)~\cite{TaoEtAl2018Scale}, real-valued Fourier phase retrieval~\cite{SinhaEtAl2017Lensless}, 3D surface tangents and normal prediction~\cite{huang2019framenet}, nonrigid structure-from-motion~\cite{kong2019deep,WangEtAl2020Deep}. We believe that our work is the first to delineate the symmetry problem  confronting effective learning and propose a solution principle that likely generalizes to other inverse problems.

Mathematically, points related by symmetries form an equivalence class and these equivalence classes form a partition of the input space for the forward model. Our symmetric breaking task effectively consists in finding a \emph{consistent} representation for the equivalence classes, where the consistency here requires the set of the representatives to be topologically connected.

\section{Numerical Experiments}
In this section, we set up various numerical experiments to verify our claim that effective symmetry breaking facilitates efficient learning. Specifically, we work with both the real and complex Gaussian PR problems, and try to answer the following questions:
\begin{itemize}
\item Do our symmetry breaking schemes help get us improved recovery performance?
\item How does the performance vary when the problem dimension $n$ changes?
\item What is the effect of the number of training samples on the recovery performance?
\end{itemize}

\subsection{Basic experimental setups}
\paragraph{Learning models}
We set up an end-to-end pipeline and use neural network models to approximate the inverse mappings, as is typical done in this approach. The following are brief descriptions of the models used in our comparative study. Recall that in our problem setup for Gaussian PR, $n$ is the dimension for $\mb x$ and $m$ is the dimension for $\mb y$.
\begin{itemize}
\item{\textbf{Neural Network (NN)}}: fully connected feedforward NN with architecture $m$-256-128-64-$n$.
\item{\textbf{Wide Neural Network (WNN)}}: we increase the size of hidden units of the NN by a factor of 2. The architecture is $m$-512-256-128-$n$.
\item{\textbf{Deep Neural Network (DNN)}}: we increase the number and size of hidden layers of the NN by adding two more layers. The architecture is $m$-2048-1024-512-256-128-$n$.
\item{\textbf{$K$-Nearest Neighbors ($K$-NN)}}: $K$-NN regression, where prediction is the average of the values of $K$ nearest neighbors. In this work we use $K = 5$.
\end{itemize}

\paragraph{Data}
We take $m = 4n$, which is just above the threshold for ensuring injectivity of the forward models up to the intrinsic symmetry~\cite{BalanEtAl2006signal}\footnote{For real Gaussian PR, the threshold is near $m = 2n$. Here, we use $m = 4n$ for both the real and complex versions for simplicity. }. To generate data samples, we draw iid uniformly random data points $\mb x_i$'s from the unit ball and consequently generate $\{(\mb x_i,\abs{\mb A \mb x_i}^2)\}$ as the simulated datasets. All the datasets are split into $80\%$ training and $20\%$ test, and $10\%$ of the training data are used for validation.

We conduct experiments with varying input dimension $n$ and dataset size. Specifically, we experiment with $n=5, 10, 15$ and dataset size of $2e4$, $5e4$, $1e5$, $1e6$, respectively. We do not test higher dimensions, in view of the exponential growth of sample requirement to cover the high-dimensional ball. For most practical inverse problems where the input data tend to possess low-dimensional structures despite the apparent high dimensions, the sample requirement will not be an issue and our symmetric breaking scheme naturally adapts to the structures. One may suggest performing a real-data experiment on natural images, as was done in numerous previous papers~\cite{CandesEtAl2012PhaseLift,CandesEtAl2015Phase,SunEtAl2017Geometric}. This is sensible, but not interesting in the Gaussian PR context, as the sign (resp. phase transfer) symmetry for real (resp. complex) PR is naturally broken due to the restriction of the image values to be nonnegative. As we argued before, the Gaussian settings erase the essential difficulties of Fourier PR. Real-data experiments will become nontrivial in Fourier PR, as the nonnegative restriction does not kill the flipping and translation symmetries. But that entails working out the specific symmetry breaking strategy for Fourier PR; we leave it as future work.

For all neural network models, we train them based on two variants of the training samples: one with symmetry untouched (i.e., before symmetry breaking) and one with symmetry breaking (i.e., after symmetry breaking). The former just leaves the samples unchanged, whereas the latter pre-processes the training samples using the procedures we described in \cref{sec:gpr} for the real- and complex-valued Gaussian PR, respectively.  To distinguish the two variants, we append our neural network model names with ``-A" to indicate \emph{after} symmetry breaking and ``-B" to indicate \emph{before} symmetry breaking.

\paragraph{Training and error metric}
The mean loss is used in the objective. We use the Adam optimizer \cite{kingma2014adam} and train all models for a maximum of $100$ epochs. The learning rate is set as $0.001$ by default and training is stopped if the validation loss does not reduce for $10$ consecutive epochs. The validation set is also used for hyperparameter tuning. To train the models for the complex PR, real and complex parts of any complex vector are concatenated into a long real vector. The $K$-Nearest Neighbor model is fit on the whole training dataset and serves as a baseline for each experiment.

To imitate the real-world test scenario, we do not perform symmetry breaking on the test data. To measure performance, we use the normalized mean square error (MSE) which is rectified to account for the symmetry:
\begin{align}\label{eq:error_metric}
    \eps_{\mathrm{real}} = & \min_{s \in \{+1,-1\}} \frac{||\widehat{\mb x} s - \mb x||^{2}}{n}, \quad & \text{(real)} \\
    \eps_{\mathrm{comp}} = & \min_{\theta \in [0,2\pi)} \frac{|| \widehat{\mb x} e^{i\theta} - \mb x||^{2}}{n}., \quad & \text{(complex)}
\end{align}
where $\widehat{\mb x}$ is the prediction by the learned models.

\subsection{Quantitative results}\label{sec:results1}
\begin{table*}[!htbp]
\caption {Summary of results in terms of test error for real Gaussian PR. Blue numbers indicate the best performance in each row.} \label{tab:NN_real}
\setlength{\tabcolsep}{4pt}
\renewcommand{\arraystretch}{1.1}
\begin{center}
{\small
\begin{tabular}{c|c|ccc|ccc|ccc}
\hline
\multirow{1}{*}{\textbf{$n$}} & \multirow{1}{*}{\textbf{Sample}}  & \multirow{1}{*}{\textbf{NN-$A$}}  & \multirow{1}{*}{\textbf{$K$-NN}}  & \multirow{1}{*}{\textbf{NN-$B$}} & \multirow{1}{*}{\textbf{WNN-$A$}}  & \multirow{1}{*}{\textbf{$K$-NN}}  & \multirow{1}{*}{\textbf{WNN-$B$}} & \multirow{1}{*}{\textbf{DNN-$A$}}  & \multirow{1}{*}{\textbf{$K$-NN}}  & \multirow{1}{*}{\textbf{DNN-$B$}} \\ \hline

\multirow{4}{*}{$5$} & 2e4 & 0.0010 & 0.0017 & 0.0283 & \bluebf{0.0008} & 0.0018 & 0.0283 & 0.0010 & 0.0019 & 0.0284 \\ \cline{2-11}
 & 5e4 & \bluebf{0.0006} & 0.0012 & 0.0282 & 0.0008 & 0.0017 & 0.0284& 0.0007 & 0.0014 & 0.0285  \\ \cline{2-11}
 & 1e5 & \bluebf{0.0005} & 0.0010 & 0.0284 & \bluebf{0.0005} & 0.0012 & 0.0283 & 0.0013 & 0.0018 & 0.0284  \\ \cline{2-11}
 & 1e6 & \bluebf{0.0004} & 0.0007 & 0.0283 & 0.0005 & 0.0006 & 0.0283  & 0.0007 & 0.0008 & 0.0283 \\\thickhline
\multirow{4}{*}{$10$} & 2e4  & 0.0011 & 0.0020 & 0.0082 & 0.0009 & 0.0022 & 0.0082 & \bluebf{0.0008} & 0.0021 & 0.0082  \\ \cline{2-11}
 & 5e4 & 0.0009 & 0.0016 & 0.0082 & \bluebf{0.0006} & 0.0018 & 0.0082 & 0.0009 & 0.0020 & 0.0082 \\ \cline{2-11}
 & 1e5 & 0.0009 & 0.0016 & 0.0082  & \bluebf{0.0006} & 0.0015 & 0.0082  & 0.0008 & 0.0017 & 0.0082  \\ \cline{2-11}
 & 1e6  & 0.0007 & 0.0013 & 0.0082 & \bluebf{0.0005} & 0.0010 & 0.0082 & 0.0009 & 0.0011 & 0.0082 \\\thickhline
\multirow{4}{*}{$15$} & 2e4 & 0.0012 & 0.0017 & 0.0038 & \bluebf{0.0009} & 0.0016 & 0.0038 & \bluebf{0.0009} & 0.0016 & 0.0038 \\ \cline{2-11}
 & 5e4  & 0.0011 & 0.0014 & 0.0038 & 0.0009 & 0.0014 & 0.0038  & \bluebf{0.0008} & 0.0015 & 0.0038 \\ \cline{2-11}
 & 1e5 &  0.0010 & 0.0013 & 0.0038  & 0.0008 & 0.0013 & 0.0038 & \bluebf{0.0007} & 0.0013 & 0.0038 \\ \cline{2-11}
 & 1e6  & 0.0008 & 0.0009 & 0.0038 & \bluebf{0.0007} & 0.0010 & 0.0038  & 0.0009 & 0.0010 & 0.0038 \\ \hline
\end{tabular}
}
\end{center}
\end{table*}

\begin{table*}[!htbp]
\caption {Summary of results in terms of test error for complex Gaussian PR. Blue numbers indicate the best performance in each row.} \label{tab:NN_complex}
\setlength{\tabcolsep}{4pt}
\renewcommand{\arraystretch}{1.1}
\begin{center}
{\small
\begin{tabular}{c|c|c c c|c c c|c c c}
\hline
\multirow{1}{*}{\textbf{$n$}} & \multirow{1}{*}{\textbf{Sample}}  & \multirow{1}{*}{\textbf{NN-$A$}}  & \multirow{1}{*}{\textbf{$K$-NN}}  & \multirow{1}{*}{\textbf{NN-$B$}} & \multirow{1}{*}{\textbf{WNN-$A$}}  & \multirow{1}{*}{\textbf{$K$-NN}}  & \multirow{1}{*}{\textbf{WNN-$B$}} & \multirow{1}{*}{\textbf{DNN-$A$}}  & \multirow{1}{*}{\textbf{$K$-NN}}  & \multirow{1}{*}{\textbf{DNN-$B$}} \\ \hline



\multirow{4}{*}{5} & 2e4 & 0.0016 & 0.0044 & 0.0786 & \bluebf{0.0011} & 0.0087 & 0.0882 & 0.0013 & 0.0045 & 0.0699 \\ \cline{2-11}

 & 5e4 & \bluebf{0.0010} & 0.0039 & 0.0718 & 0.0012 & 0.0038 & 0.0669 & 0.0019 & 0.0039 & 0.0697 \\ \cline{2-11}

& 1e5 & \bluebf{0.0006} & 0.0021 & 0.0473 & 0.0032 & 0.0034 & 0.0942 & 0.0011 & 0.0013 & 0.0854\\ \cline{2-11}

& 1e6 & \bluebf{0.0005} & 0.0006 & 0.0642 & 0.0064 & 0.0072 & 0.0453 & 0.0014 & 0.0015 & 0.0731 \\\thickhline

\multirow{4}{*}{10} & 2e4 & 0.0079 & 0.0237 & 0.0452 & 0.0065 & 0.0080 & 0.0453  & \bluebf{0.0061} & 0.0239 & 0.0380 \\ \cline{2-11}

 & 5e4  & \bluebf{0.0056} & 0.0066 & 0.0428 & 0.0089 & 0.0191 & 0.0419  & 0.0082 & 0.0181 & 0.0400 \\ \cline{2-11}

 & 1e5  & 0.0097 & 0.0139 & 0.0436 & \bluebf{0.0028} & 0.0058 & 0.0431 & 0.0055 & 0.0086 & 0.0453 \\
 \cline{2-11}
  & 1e6 & 0.0136 & 0.0179 & 0.0448 & 0.0085 & 0.0162 & 0.0432 & \bluebf{0.0077} & 0.0118 & 0.0399 \\\thickhline

\multirow{4}{*}{15} & 2e4  & 0.0282 & 0.0287 & 0.0282 & \bluebf{0.0129} & 0.0143 & 0.0296 & 0.0180 & 0.0189 & 0.0277 \\ \cline{2-11}

 & 5e4  & 0.0192 & 0.0272 & 0.0313  & \bluebf{0.0062} & 0.0126 & 0.0308  & 0.0172 & 0.0233 & 0.0313 \\ \cline{2-11}
 & 1e5   & 0.0188 & 0.0226 & 0.0258  & \bluebf{0.0141} & 0.0269 & 0.0295  & 0.0177 & 0.0206 & 0.0274 \\ \cline{2-11}

 & 1e6    & 0.0136 & 0.0179 & 0.0448  & \bluebf{0.0131} & 0.0184 & 0.0395
 & 0.0182 & 0.0202 & 0.0283 \\ \hline
\end{tabular}
}
\end{center}
\end{table*}

\cref{tab:NN_real} provides test errors for all models trained for real PR, and likewise \cref{tab:NN_complex} presents test errors for complex PR. All models for the same combination of input dimension $n$ and sample size use the same set of data. Blues numbers in the tables indicate the best performing model across all the models in each row.

We first note that for the same NN architecture with any dimension-sample combination, symmetry breaking always leads to substantially improved performance. Without symmetry breaking, i.e., as shown in the $(\cdot)$-$B$ columns, the estimation errors are always worse, if not significantly so, than the simple baseline $K$-NN model. By contrast, symmetry breaking as shown in the $(\cdot)$-$A$ columns always leads to improved performance compared to the baseline. To rule out the possibility that the inferior performance of $(\cdot)$-B's is due to the capacity of the NNs, we can compare the performance of NN with that of WNN and DNN, where the latter two have $4\times$ and $50\times$ more parameters than the plain NN, as shown in \cref{tab:nnparam}.
\begin{table}[!htbp]
	\caption {Count of trainable parameters for $n=15$} \label{tab:nnparam}
	\centering
		\begin{tabular}[t]{ccc}
		\toprule
		Models & Real  & Complex \\
		\midrule
		Neural Network   & 57,743  & 58,718  \\
		Wide Neural Network &   197,391 & 199,326\\
		Deep Neural Network & 2,914,063 & 2,915,998\\
		\bottomrule
	\end{tabular}
\end{table}
From \cref{tab:NN_real} and \cref{tab:NN_complex}, it is clear that the large capacities do not lead to improved performance, suggesting that our plain NNs are probably already powerful enough. Moreover, for a fixed learning model, increasing the number of samples beyond $2e4$ also only yields marginally improved errors, indicating that bad performance cannot be attributed to lack of samples. These observations together show that $(\cdot)$-$B$'s are inefficient learners, as they do not explicitly handle the symmetry problem.

Moreover, as the dimension $n$ grows, there is a persistent trend that the $\paren{\cdot}$-B models performs incrementally better. This might be counter-intuitive at first sight, as the coverage of the space becomes sparser as the dimension grows with same number of random samples. A reverse trend could be expected. But as we hinted at the end of \cref{sec:intro_sym_problem}, more training samples also cause more wildly behaved functions---the problem becomes less severe as the dimension grows as the density of sample points becomes smaller. In fact, when the sample density is extremely low, the other trend that is dictated by the lack of samples could reveal. Nonetheless, here we focus on the data-intensive regime. Overall, the difficulty of approximating highly oscillatory functions is evident.

\subsection{Difficulty with symmetries: what happened?}
 In this section, we investigate several aspects of the neural networks in the hope that some aspect can potentially help overcome the learning difficulties with symmetries. Based on the above discussion, we focus on the NN-$A$ model. Typically, besides the network size, performance of neural networks is also strongly affected by the mini-batch size, learning rate, and regularization. To analyze the impact of the latter three, we vary each one of the parameters while keeping the others fixed. We work with real PR only and expect the situation for complex PR to be similar. To keep a reasonably fast run time while not hurting the performance, we take $2e4$ data samples, which seems sufficient for the above results.

 \begin{figure}[!htbp]
 	\centering
 	\includegraphics[width=0.45\linewidth]{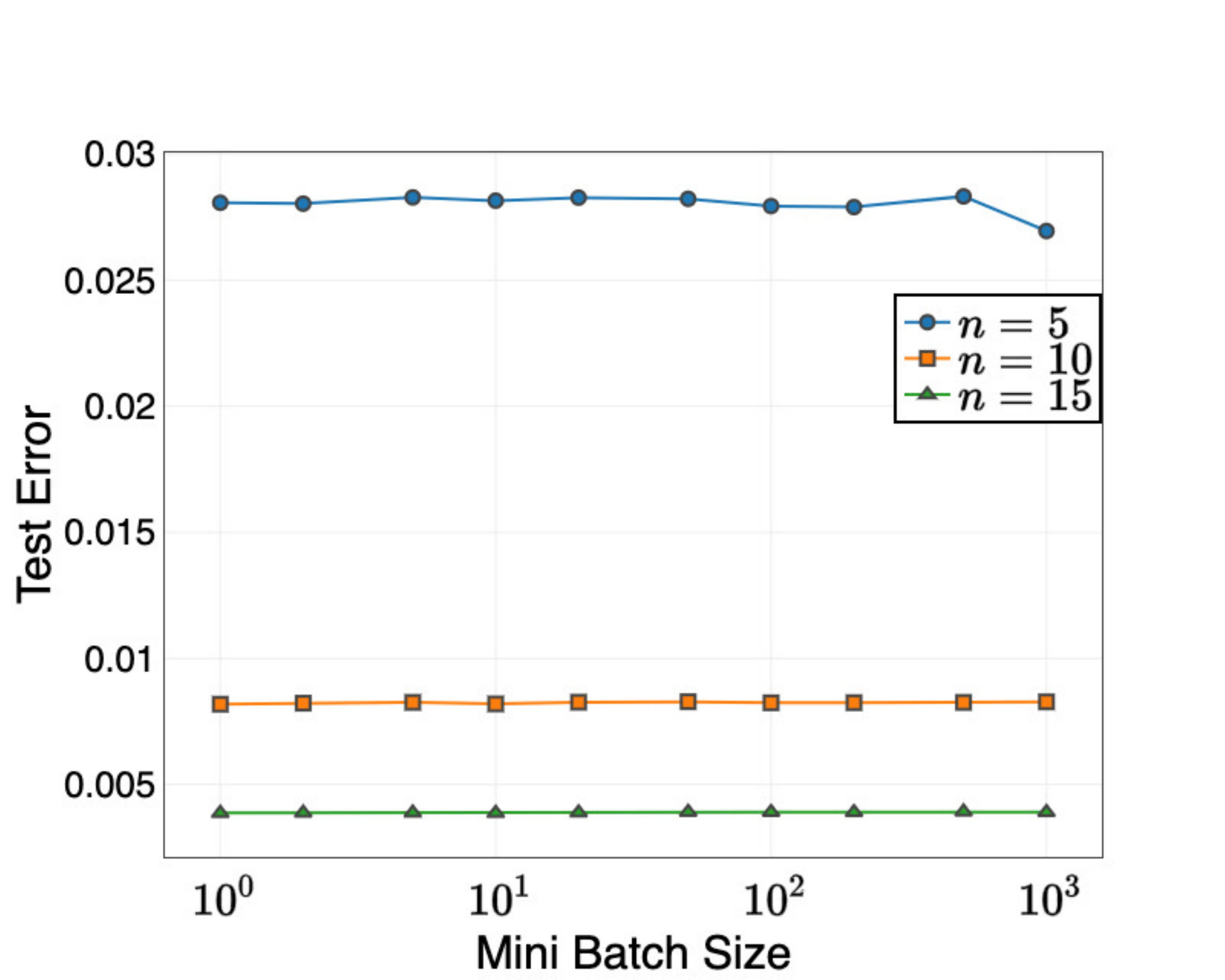}
   \includegraphics[width = 0.45\linewidth]{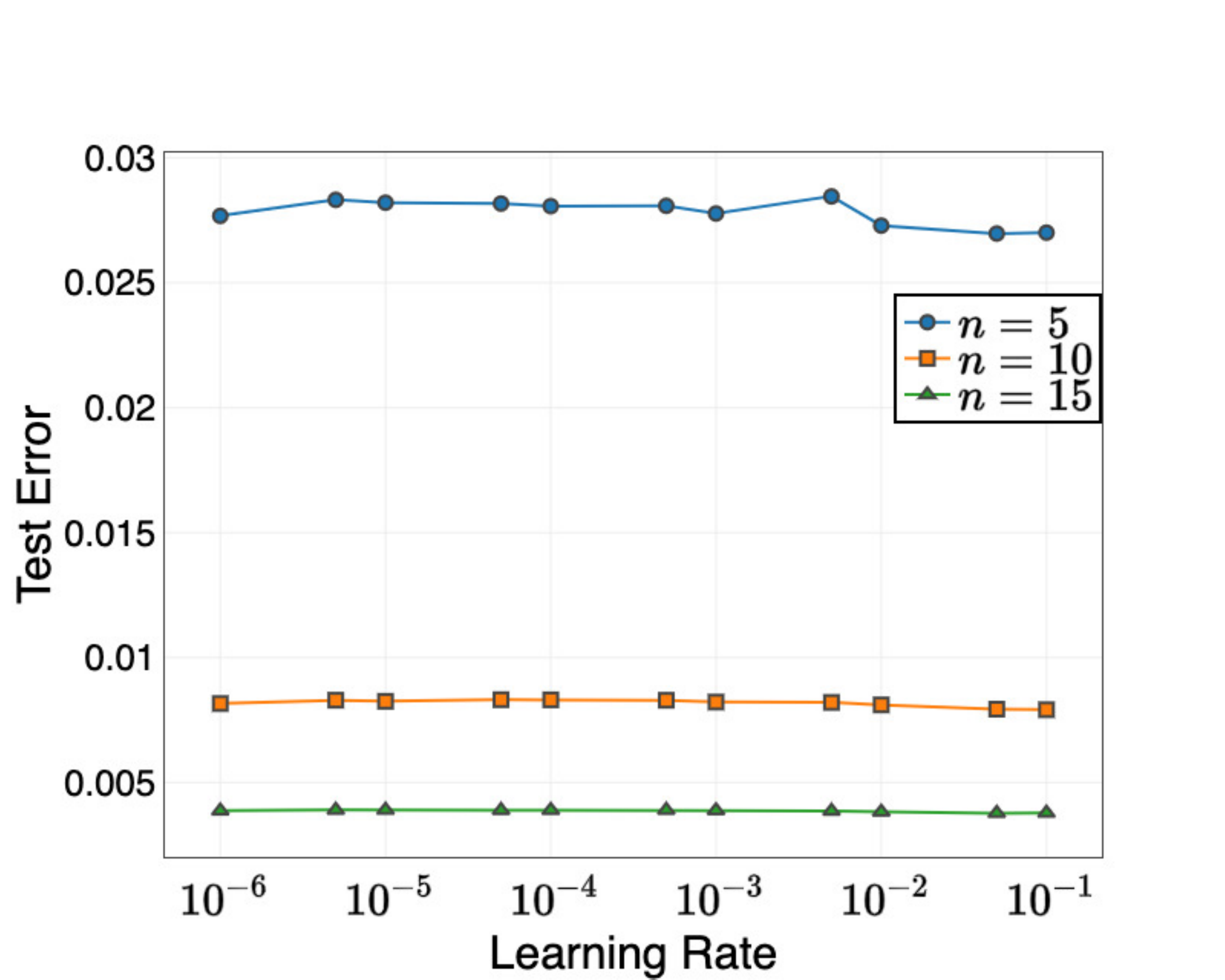}
 	\caption{(Left) Test error vs. mini-batch size; (Right) Test error vs. learning rate. }
 	\label{fig:minibatch}
 \end{figure}
\paragraph{Effect of mini-batch size}
The mini-batch size in stochastic optimization algorithms such as Adam that we use is considered to have a substantial impact on the performance of neural networks \cite{bengio2012practical}. To see if this can help with the performance, we change the size to sweep several orders of magnitudes, i.e., $1e0, 1e1, 1e2, 1e3$, and also experiment with different dimensions, i.e., $n = 5, 10, 15$, on NN-$B$. From the results presented in \cref{fig:minibatch} (Left), we conclude that varying the mini-batch size has a negligible effect on the test error.

\paragraph{Effect of learning rate}
The learning rate is the most critical hyper-parameter \cite{jacobs1988increased} that guides the change in model weights in response to the estimated error. To examine its effect on the test error, we vary it across six orders of magnitude: $10^{-1}$, $10^{-2}$, $10^{-3}$, $10^{-4}$, $10^{-5}$, $10^{-6}$, and retrain NN-$B$. Again, the magnitude of the test error roughly remains the same across the distinct learning rates, as shown in \cref{fig:minibatch} (Right).

\paragraph{Effect of regularization}
We explore three regularization schemes, $\mathcal{L}_1$, $\mathcal{L}_2$ and $\mathcal{L}_1$ + $\mathcal{L}_2$ . \cref{tab:reg} shows the results after our retraining of NN-$B$ with the different schemes. It appears that no scheme clearly wins out.
\begin{table}[!htbp]
	\caption {Test error using different regularization schemes} \label{tab:reg}
	\centering
		\begin{tabular}[t]{cccc}
		\toprule
		Regularization & $n = 5$  & $n = 10$ & $n = 15$ \\
		\midrule
		$\mathcal{L}_1$   & 0.02848  & 0.00831 & 0.00392  \\
		$\mathcal{L}_2$  &   0.02847 & 0.00830 & 0.00392\\
		$\mathcal{L}_1$ + $\mathcal{L}_2$ &  0.02846 & 0.00830 &  0.00392\\
		\bottomrule
	\end{tabular}
\end{table}

These results reinforce our claim that the bad performance of neural network learning without symmetry breaking is due to the intrinsic difficulty of approximating irregular functions, not due to suboptimal choice of neural network architecture or training hyper-parameters.


\section{Conclusion}
In this paper, we explain how symmetries in the forward processes can lead to difficulty---approximating highly oscillatory functions---in solving the resulting inverse problems by an end-to-end deep learning approach. Using the real and complex Gaussian PR problems as examples, we show how effective symmetry breaking can be performed to remove the above difficulties in learning, and we also verify the effectiveness of our scheme using extensive numerical experiments. In particular, we show through experiments that without carefully dealing with the symmetries, learning can be highly inefficient and the performance can be inferior to simple baseline methods.

We also identify a basic principle for breaking symmetry and phrase the task as finding connected representative set for equivalence classes. The task seems highly generic and only pertains to the certain topological and geometrical structure of the data space. This favorably suggests that our strategy is probably universal and can be adapted for other inverse problems.

{\small
\bibliographystyle{amsalpha}
\bibliography{DL4INV}
}

\end{document}